\documentclass[runningheads,orivec]{llncs}
\usepackage[utf8]{inputenc}
\usepackage[T1]{fontenc}
\usepackage{acronym}
\usepackage{dsfont}
\usepackage{color}
\usepackage{epsfig}
\usepackage{graphics}
\usepackage{graphicx}
\usepackage{epstopdf}
\usepackage{url}
\usepackage{pslatex}
\usepackage{times}
\usepackage[english,ruled]{algorithm2e}
\usepackage{newtxtext,newtxmath}

\usepackage[mode=buildnew]{standalone}
\usepackage{tikz}
\usetikzlibrary{positioning}
\usetikzlibrary{automata, positioning, arrows}
\usepackage{varwidth}
\usetikzlibrary{shapes.geometric}
\usetikzlibrary{arrows.meta}
\usetikzlibrary{arrows}
\usetikzlibrary{arrows,decorations.markings}
\usetikzlibrary{shapes.misc, positioning}

\usepackage{listings}
\lstdefinelanguage{its}
{morekeywords={if, =, +, >=, +, -, (, ), [, ], true, false, typedef, transition, int,GAL, abort, !, \{, \}, label, ", &&, ., composite, synchronization,for,gal, array},
morecomment=[l]{//},
sensitive=false,
}
\newcommand{\nat}%
{\ensuremath{\mathds{N}}}
\newcommand{\zrel}%
{\ensuremath{\mathds{Z}}}
\newcommand{\rea}%
{\ensuremath{\mathds{R}}}
\newcommand{\bool}%
{\ensuremath{\mathds{B}}}

\newcommand{\pre}%
{\ensuremath{\preccurlyeq}}

\usepackage{accents}
\newcommand{\ubar}[1]{\underaccent{\bar}{#1}}

\newcommand{\tuple}[1]{\langle #1 \rangle}

\newcounter{rrule}

\renewcommand{\implies}{\Rightarrow}

\newcommand{\lang}{\ensuremath{\mathscr{L}}}

\newcommand{\AP}{\ensuremath{\mathrm{AP}}}

\newcommand{\KS}{\ensuremath{\mathit{KS}}}

\newcommand{\tr}{\xrightarrow}

\usepackage{booktabs}
\usepackage{subcaption}

\begin{document}
\title{LTL under reductions with weaker conditions than stutter invariance}




%
%
\author{Emmanuel Paviot-Adet\inst{1,2}, Denis Poitrenaud \inst{1,2}, Etienne Renault\inst{3}, Yann Thierry-Mieg\inst{1} }

%
\authorrunning{PPRT} 

%
\institute{Sorbonne Université, CNRS, LIP6, F-75005 Paris, France \\
\email{first.last@lip6.fr}
\and Université de Paris, F-75006 Paris, France \\
 \and EPITA, LRDE, Kremlin-Bic\^etre, France\\ 
\email{ renault@lrde.epita.fr }
}
\maketitle              
\begin{abstract}

Verification of properties  expressed as $\omega$-regular languages such as LTL can benefit hugely from stutter insensitivity, using a diverse set of reduction strategies. However properties that are not stutter invariant, for instance due to the use of the neXt operator of LTL or to some form of counting in the logic, are not covered by these techniques in general.

We propose in this paper to study a weaker property than stutter insensitivity.
In a stutter insensitive language both adding and removing stutter to
a word does not change its acceptance, any stuttering can be abstracted away; by decomposing this equivalence relation into two implications we obtain weaker conditions.
We define a shortening insensitive language where any word that stutters less than a word in the language must also belong to the language. A lengthening insensitive language has the dual property.
A semi-decision procedure is then introduced to reliably prove shortening insensitive properties or deny lengthening insensitive properties while working with a \emph{reduction} of a system. A reduction has the property that it can only shorten runs. 
Lipton's transaction reductions or Petri net agglomerations are examples of eligible structural reduction strategies.

An implementation and experimental evidence is provided showing most non-random properties sensitive to stutter are actually shortening or lengthening insensitive. 
Performance of experiments on a large (random) benchmark from the model-checking competition indicate that despite being a semi-decision procedure, the approach can still improve state of the art verification tools.

\end{abstract}

\section{Introduction}


Model checking is an automatic verification technique for proving the correctness of   systems that have finite state abstractions. 
Properties can be expressed using the popular Linear-time Temporal Logic (LTL).
To verify LTL properties, the automata-theoretic approach~\cite{Vardi07} builds a product between a Büchi automaton representing the negation of the LTL formula and the reachable state graph of the system (seen as a set of infinite runs).
This approach has been used successfully to verify both hardware and software components, but it suffers from the so called "state explosion problem": as the number of state variables in the system increases, the size of the system state space grows exponentially. 

One way to tackle this issue is to consider \textit{structural reductions}. Structural reductions take their roots  in the work of  Lipton~\cite{Lipton75} and  Berthelot~\cite{Berthelot85}. Nowadays, these reductions are still considered as an attractive way to alleviate the state explosion problem~\cite{Laarman18,berthomieu19}. 
Structural reductions strive to fuse structurally "adjacent" events into a single atomic step, leading to less interleaving of independent events and less observable behaviors
in the resulting system. An example of such a structural reduction is shown on Figure~\ref{fig:example} where actions are progressively grouped (see Section~\ref{ss:kripke} for a more detailed presentation). It can be observed that the Kripke structure representing the state space of the program is significantly simplified.

%
%

Traditionally structural reductions construct a smaller system that
preserves properties such as deadlock freedom, liveness, reachability~\cite{HPP06}, and \emph{stutter insensitive} temporal logic~\cite{PPP00} such as LTL$_{\setminus X}$. 
The verification of a \emph{stutter insensitive} property on a given system
does not depend on whether non observable events (i.e. that do not update atomic propositions)
are abstracted or not.
On Fig~\ref{fig:example} both instructions "$z=40;$" and "$chan.send(z)$" of thread $\beta$ are non observable.

This paper shows that structural reductions can in fact be used even for fragments of LTL that are \emph{not} stutter insensitive. 
We identify two fragments that we call 
\emph{shortening insensitive} (if a word is in the language, any version that stutters less also) or \emph{lengthening insensitive} (if a word is in the language, any version that stutters more also).
Based on this classification we introduce two semi-decision procedures that provide a reliable verdict only in one direction: e.g. presence of counter examples is reliable for lengthening insensitive properties, but absence is not.

The paper is structured as follows, Section~\ref{sec:defs} presents the definitions and notations relevant to our setting in an abstract manner, focusing on the level of description of a language. Section~\ref{sec:red} instantiates these definitions in the more concrete setting of LTL verification. Section~\ref{sec:perfs} provides experimental evidence supporting the claim that the method is both 
applicable to many formulae and can significantly improve state of the art model-checkers. Some related work is presented in Section~\ref{sec:related} before concluding.

\begin{figure}[tbp]
    \centering
  \begin{subfigure}[t]{\linewidth}
  \centering
\tikzset{
  ->, 
  >=stealth, 
  initial text=$ $,
  every state/.style={fill=gray!10,  inner sep=3pt,minimum size=1pt},
  node distance=3cm
}

\begin{tikzpicture}

   \begin{scope} [xshift=-1.6cm, yshift=0cm]
   \node (th1)  at (2.3,0.3) {\textnormal{Thread $\alpha$}};
   \node (th1)  at (5.,0.3) {\textnormal{Thread $\beta$}};

   \path[-] (3.7,0.3) edge (3.7,-1.5);

   \node[fill=white] (th1)  at (3.7,-0.25) {\textit{Initially, x = y = z = 0}};
   \node (th1)  at (1.75,-0.8) { x++};
   \node (th1)  at (2.4,-1.2) {y=chan.recv()};

   \node (th1)  at (4.3,-0.8) {z=40};
   \node (th1)  at (4.85,-1.2) {chan.send(z)};

   \path[-] (1.2,0) edge (6.2,0);
   \path[-,dashed] (1.2,-0.5) edge (6.2,-0.5);

   \begin{scope}[xshift=0cm,yshift=-.2cm]
     \draw[rounded corners,fill=lightgray!10] (1.2, -1.6) rectangle (6.5, -2.6) {};

     \node[] (th1)  at (3.7,-1.9) { \textit{Atomic propositions: p $\gets$ x $\stackrel{?}{=}$ 0}};
     \node[] (th1)  at (3.7,-2.3) {\small \textit{\hspace{22.2ex} q $\gets$ y $\stackrel{?}{=}$ 0}};

   \end{scope}
   
    \node (th1)  at (3.7,-3.6) {(1) Program};

     \node (0)  at (.9,-.68) {\textcolor{lightgray}{\textit{\tiny{ $\alpha_0\rightarrow$}}}};
     \node (0)  at (.9,-1.03) {\textcolor{lightgray}{\textit{\tiny{  $\alpha_1\rightarrow$}}}};
     \node (0)  at (.9,-1.35) {\textcolor{lightgray}{\textit{\tiny{  $\alpha_2\rightarrow$}}}};

     \node (0)  at (6.7,-.68) {\textcolor{lightgray}{\textit{\tiny{$\leftarrow\beta_0$}}}};
     \node (0)  at (6.7,-1.03) {\textcolor{lightgray}{\textit{\tiny{$\leftarrow\beta_1$}}}};
     \node (0)  at (6.7,-1.35) {\textcolor{lightgray}{\textit{\tiny{$\leftarrow\beta_2$}}}};

   \end{scope}

   \begin{scope}[xshift=7cm,yshift=0cm]
   \node[state, initial, initial left] (s0) at  (0,-1.5)   {$p q$};
   \node[state]                         (s1) at  (0, 0) {$\bar pq$};
   \node[state]                         (s2) at  (2.2, 0) {$\bar pq$};
   \node[state]                         (s3) at  (4.4, 0) {$\bar pq$};
   \node[state]                         (s4) at  (2.2, -1.5) {$pq$};
   \node[state]                         (s5) at  (4.4, -1.5) {$pq$};
   \node[state]                         (s6) at  (6.6, -1.5) {$\bar p \bar q$};
   \draw   (s0) edge[] node[fill=white]{{x++}} (s1);
   \draw   (s1) edge[bend left] node[above]{{z=40}} (s2);
   \draw   (s2) edge[bend left] node[above]{{chan.send(z)}} (s3);
   \draw   (s0) edge[bend right] node[below]{{z=40}} (s4);
   \draw   (s4) edge[] node[fill=white]{{x++}} (s2);
   \draw   (s4) edge[bend right] node[below]{{chan.send(z)}} (s5);
   \draw   (s3) edge[bend left, below] node[fill=white]{\hspace{3ex}{y=chan.recv()}} (s6);
    \draw   (s5) edge[] node[fill=white]{{x++}} (s3);

    \node (th1)  at (2.6,-3.) {$\lang_A = \{$ $pq^3$ $\bar p q$ $\bar p \bar q ^\omega$, 
      $pq^2$ $\bar p q^2$ $\bar p \bar q ^\omega$,
      $pq$ $\bar p q^3$ $\bar p \bar q ^\omega$
    $\} $};

    \node (th1)  at (2.4,-3.6) {(2) State space if each statement of the program is atomic};

   \end{scope}

   \begin{scope}[xshift=0cm,yshift=-5cm]
   \node[state, initial, initial left] (s0) at  (0,-1.5)   {$p q$};
   \node[state]                        (s1) at  (0, 0) {$\bar pq$};
   \node[state]                        (s3) at  (3.4, 0) {$\bar pq$};
   \node[state]                         (s5) at  (3.4, -1.5) {$pq$};
   \node[state]                         (s6) at  (5.6, -1.5) {$\bar p \bar q$};
   \draw   (s0) edge[] node[fill=white]{{x++}} (s1);
   \draw   (s1) edge[red, dashed, bend left,below] node[]{{\hspace{2ex} \shortstack[l]{z=40;\\chan.send(z)}}} (s3);
   \draw   (s0) edge[red, dashed, bend right,above] node[]{{\hspace{2ex} \shortstack[l]{z=40;\\chan.send(z)}}} (s5);
   \draw   (s3) edge[bend left, below] node[fill=white]{\hspace{3ex}{y=chan.recv()}} (s6);
      \draw   (s5) edge[] node[fill=white]{{\hspace{-2ex} x++}} (s3);

       \node (th1)  at (2.1,-3.) {$\lang_B =\{$ $pq$ $\bar p q^2$ $\bar p \bar q ^\omega$, $pq^2$ $\bar p q$ $\bar p \bar q ^\omega$ $\}$ };

    \node (th1)  at (2.1,-3.6) {(3) State space if "z=40;send(z)" is atomic.};

   \end{scope}

   \begin{scope}[xshift=8.5cm,yshift=-5cm]
   \node[state, initial, initial left] (s0) at  (0,-1.5)   {$p q$};
   \node[state]                        (s1) at  (0, 0) {$\bar pq$};
   \node[state]                         (s6) at  (4.6, -1.5) {$\bar p \bar q$};
   \draw   (s0) edge[] node[fill=white]{{x++}} (s1);
   \draw   (s1) edge[red, thick, dashed, bend left, below] node[red, fill=none,align=left]{{\hspace{2ex} \shortstack[l]{z=40;\\chan.send(z);\\y=chan.recv()}}} (s6);
   
       \node (th1)  at (2.1,-3.) {$\lang_C = \{$  $pq$ $\bar p q$ $\bar p \bar q ^\omega$ $\}$ };

     \node (th1)  at (2.4,-3.6) {(4) State space if "z=40;send(z);y=chan.recv()" is atomic.};
 
   \end{scope}

\end{tikzpicture}
    \caption{
    Example of reductions. (1) is a program with two threads and 3 variables. \textit{chan} is a communication channel where \textit{send(int)} insert a message and \textit{int recv()} waits until a message is available and then consumes it. We consider that the logic only observes whether $x$ or $y$ is zero denoted $p$ and $q$. (2) depicts the state-space represented as a Kripke structure. Each node is labelled by the value of atomic propositions $p$ and $q$. 
    When an instruction is executed the value of these propositions \emph{may} evolve. (3) represents the state-space of a structurally reduced version of the program where actions of thread $\beta$ "z=40;chan.send(z)" are fused into a single atomic operation. (4) represents the state-space of a program where the three actions of the original program "z=40;chan.send(z);y=chan.recv()" are now a single atomic step.
    }
    \label{fig:example}
\end{subfigure}
\begin{subfigure}[t]{\linewidth}
    \centering
\includestandalone[width=0.63\textwidth]{figs/figure}
    \caption{
    $\Sigma^\omega$ is represented as a circle that is partitioned into equivalence classes of words ($\hat{r_0}, \hat{r_1} \ldots$). Each point in the space is a word, and some of the $\pre$ relations are represented as arrows ; the red point is the shortest word $\underline{\hat{r}}$ in the equivalence class. Gray areas are inside the language, white are outside of it. Four languages are depicted : 
   $\lang_a$: equivalence classes are entirely inside or outside \textbf{a stutter insensitive language}, 
    $\lang_b$: the "bottom" of an equivalence class may belong to \textbf{a shortening insensitive language},  
    $\lang_c$: the "top" of an equivalence class may belong to \textbf{a lengthening insensitive language},  
    $\lang_d$: some languages are neither lengthening insensitive nor shortening insensitive.
    }
    \label{fig:lang}
\end{subfigure}
\end{figure}

\section{Definitions}
\label{sec:defs}

In this section we first introduce in Section~\ref{ss:shorterthan} a "shorter than" partial order relation on infinite words, based on the number of repetitions or stutter in the word. 
This partial order gives us in Section~\ref{ss:shortlong} the notions of shortening and lengthening insensitive language, which are shown to be weaker versions of classical stutter insensitivity in Section~\ref{ss:relationSI}.
We then define in Section~\ref{ss:checkred} the \emph{reduction of a language} which contains a shorter representative of each word in the original language. 
Finally we show that we can use a semi-decision procedure to verify shortening or lengthening insensitive properties using a reduction of a system.

\subsection{A "Shorter than" relation for infinite words}
\label{ss:shorterthan}

\begin{definition}
    [Word]: A word over a finite alphabet $\Sigma$ is
    an infinite sequence of symbols in $\Sigma$.
    We canonically denote a word $r$ using one of the two forms:  
    \begin{itemize}
        \item (plain word) $r=a_0^{n_0}a_1^{n_1}a_2^{n_2}\ldots$ with for all $i \in \nat$, $a_i \in \Sigma$, $n_i \in \nat^+$ and $a_i \neq a_{i+1}$, or
        \item ($\omega$-word) $r=a_0^{n_0}a_1^{n_1}\ldots a_k^\omega$ with $k \in \nat$ and for all $0 \leq i \leq k$, $a_i \in \Sigma$, and for $i<k$, $n_i \in \nat^+$ and $a_i \neq a_{i+1}$. $a_k^\omega$ represents an infinite stutter on the final symbol $a_k$ of the word.
    \end{itemize}
    The set of all words over alphabet $\Sigma$ is denoted $\Sigma^\omega$.
\end{definition}

These notations using a power notation for repetitions of a symbol in a word are introduced to highlight stuttering. We force the symbols to alternate to ensure we have a canonical representation: with $\sigma$ a suffix (not starting by symbol $b$), the word $aab\sigma$ must be represented as $a^2b^1\sigma$ and not $a^1a^1b^1\sigma$. 
To represent a word of the form $aabbcccccc\ldots$ we use an $\omega$-word: $a^2b^2c^\omega$.

\begin{definition}
    [Shorter than]: A plain word  $r=a_0^{n_0}a_1^{n_1}a_2^{n_2}\ldots$ is \emph{shorter} than a plain word $r'=a_0^{n_0'}a_1^{n_1'}a_2^{n_2'}\ldots$ if and only if for all $i \in \nat$, $0 < n_i \leq n_i'$. 
    For two $\omega$-words $r=a_0^{n_0}\ldots a_k^\omega$ and $r'=a_0^{n_0'}\ldots a_k^\omega$, $r$ is \emph{shorter} than $r'$ if and only if for all $i < k$, $n_i \leq n_i'$.
    
    We denote this relation on words as $r \pre r'$. 
\end{definition}

For instance, for any given suffix $\sigma$, $a b \sigma \pre a^2 b \sigma$. Note that $a b \sigma \pre a b^2 \sigma$ as well, but that $a^2 b \sigma$ and $a b^2 \sigma$ are incomparable. $\omega$-words are incomparable with plain words. 

\begin{property}
    The $\pre$ relation is a partial order on words. 
\end{property}

\begin{proof}
The relation is clearly reflexive ($\forall r \in \Sigma^\omega, r \pre r$), anti-symmetric ($\forall r, r' \in \Sigma^\omega$, $r \pre r' \land r' \pre r \implies r = r'$) and transitive ($\forall r, r',r'' \in \Sigma^\omega, r \pre r' \land r' \pre r'' \implies r \pre r''$). The order is partial since some words (such as $a^2 b \sigma$ and $a b^2 \sigma$ presented above) are incomparable.
\qed
\end{proof}

\begin{definition}
\label{def:stutterequiv}
    [Stutter equivalence]: a word  $r$ is \emph{stutter equivalent} to $r'$, denoted as $r \sim r'$ if and only if there exists a shorter word $r''$ such that $r'' \pre r \land r'' \pre r'$. This relation $\sim$ is an equivalence relation thus partitioning words of $\Sigma^\omega$ into equivalence classes.
    
    We denote $\hat{r}$ the equivalence class of a word $r$ and denote $\ubar{\hat{r}}$ the shortest word in that equivalence class.
\end{definition}

For any given word $r=a_0^{n_0}a_1^{n_1}a_2^{n_2}\ldots$ there is a shortest representative in $\hat{r}$ that is the word $\ubar{\hat{r}}=a_0 a_1 a_2\ldots$ where no symbol is ever consecutively repeated more than once (until the $\omega$ for an $\omega$-word). By definition all words that are comparable to $\ubar{\hat{r}}$ are stutter equivalent to each other, since $\ubar{\hat{r}}$ can play the role of $r''$ in the definition of stutter equivalence, giving us an equivalence relation: it is reflexive, symmetric and transitive.

For instance, with $\sigma$ denoting a suffix, $\ubar{\hat{r}}=a b \sigma$ would be the shortest representative of any word of $\hat{r}$ of the form $a^{n_0}b^{n_1}\sigma$. We can see by this definition that despite being incomparable, $a^2 b \sigma \sim a b^2 \sigma$ since $a b \sigma \pre a^2 b \sigma$ and $a b \sigma \pre a b^2 \sigma$.

\subsection{Sensitivity of a language to the length of words}
\label{ss:shortlong}

\begin{definition}
    [Language]: a language $\lang$ over a finite alphabet $\Sigma$ is a set of words over $\Sigma$, hence $\lang \subseteq \Sigma^\omega$. We denote $\bar{\lang} =  \Sigma^\omega \setminus \lang$ the complement of a language $\lang$.
\end{definition}

In the literature, most studies that exploit a form of stuttering are focused on \emph{stutter insensitive} languages \cite{porBook,Valmari90,Peled94,GodefroidW94,HPP06}.
 In a stutter insensitive language $\lang$, duplicating any letter (also called stuttering) or removing any duplicate letter from a word of $\lang$ must produce another word of $\lang$.
In other words, all stutter equivalent words in a class $\hat{r}$ must be either in the language or outside of it. Let us introduce weaker variants of this property, original in this paper.

\begin{definition}
\label{def:shortins}
    [Shortening insensitive]: a language $\lang$ is \emph{shortening insensitive} if and only if for any word $r$ it contains, all shorter words $r'$ such that $r' \pre r$ are also in $\lang$. 
\end{definition}

For instance, a shortening insensitive language $\lang$ that contains the word $a^3 b \sigma$ must also contain shorter words $a^2 b \sigma$, and $a b \sigma$. If it contains $a^2 b^2 \sigma$ it also contains $a^2 b \sigma$, $a b^2 \sigma$ and $a b \sigma$.

\begin{definition}
\label{def:longins}
    [Lengthening insensitive]: a language $\lang$ is \emph{lengthening insensitive} if and only if for any word $r$ it contains, all longer words $r'$ such that $r \pre r'$ are also in $\lang$. 
\end{definition}

For instance, a lengthening insensitive language $\lang$ that contains the word $a^2 b \sigma$ must also contain all longer words $a^3 b \sigma$, $a^2 b^2 \sigma$ \ldots, and more generally words of the form $a^{n} b^{n'} \sigma$ with $n \geq 2$ and $n' \geq 1$. If it contains $\ubar{\hat{r}}=a b \sigma$ the shortest representative of an equivalence class, it contains all words in the stutter equivalence class.

While stutter insensitive languages have been heavily studied, there is to our knowledge no study on what reductions are possible if only one direction holds, i.e. the language is shortening or lengthening insensitive, but not both. A shortening insensitive language is essentially asking for something to happen \emph{before} a certain deadline or stuttering "too much". A lengthening insensitive language is asking for something to happen \emph{at the earliest} at a certain date or after having stuttered at least a certain number of times. Figure~\ref{fig:lang} represents these situations graphically.


\subsection{Relationship to stutter insensitive logic}
\label{ss:relationSI}

A language is both shortening and lengthening insensitive if and only if it is stutter insensitive (see Fig.~\ref{fig:lang}).
This fact is already used in~\cite{ADL15} to identify such stutter insensitive languages using only their automaton.
Furthermore since stutter equivalent classes of runs are entirely inside or outside a stutter insensitive language, 
a language $\lang$ is stutter insensitive if and only if the complement language $\bar{\lang}$ is stutter insensitive.

However, if we look at sensitivity to length and how it interacts with the complement operation, we find a dual relationship where the complement of a shortening insensitive language is lengthening insensitive and vice versa. 

\begin{property}
\label{prop:shortduallong}
A language $\lang$ is shortening insensitive if and only if the complement language $\bar{\lang}$ is lengthening insensitive.
\end{property}

\begin{proof}
Let $\lang$ be shortening insensitive. Let $r \in \bar{\lang}$ be a word in the complement of $\lang$. Any word $r'$ such that $r \pre r'$ must also belong to $\bar{\lang}$, since if it belonged to the shortening insensitive $\lang$, $r$ would also belong to $\lang$. Hence $\bar{\lang}$ is lengthening insensitive. The converse implication can be proved using the same reasoning.
\qed
\end{proof}

If we look at Figure~\ref{fig:lang}, the dual effect of complement on the sensitivity of the language to length is  apparent: if gray and white are switched we can see $\bar{\lang_b}$ is lengthening insensitive and $\bar{\lang_c}$ shortening insensitive.

\subsection{When is visiting shorter words enough?}
\label{ss:checkred}

\begin{definition}
\label{def:reduction}
    [Reduction] 
    Let $I$ be a reduction function $\Sigma^\omega \mapsto \Sigma^\omega$ such that $\forall r, I(r) \pre r$. The reduction by $I$ of a language $\lang$ is $Red_I(\lang)=\{I(r) \mid r \in \lang\}$. 
\end{definition}


Note that the $\pre$ partial order is not strict so that the image of a word may be the word itself, hence identity is a reduction function. In most cases however we expect the reduction function to map many words $r$ of the original language to a single shorter word $r'$ of the reduced language. 
Note that given any two reduction functions $I$ and $I'$, $Red_I(Red_{I'}(\lang))$ is still a reduction of $\lang$. Hence chaining reduction rules still produces a reduction.
As we will discuss in Section~\ref{ss:kripke} structural reductions of a specification such as Lipton's transaction reduction~\cite{Lipton75,Laarman18} or Petri net agglomerations~\cite{Berthelot85,YTM20} induce a reduction at the language level.
In Fig.~\ref{fig:example} fusing statements into a single atomic step in the program induces a reduction of the language.

\begin{theorem}
[Reduced Emptiness Checks]
\label{th:short}
Given two languages $\lang$ and $\lang'$, 
\begin{itemize}
    \item if $\lang$ is shortening insensitive, then $\lang \cap Red_I(\lang') = \emptyset \implies \lang \cap \lang' = \emptyset$
    \item if $\lang$ is lengthening insensitive, then $\lang \cap Red_I(\lang') \neq \emptyset \implies \lang \cap \lang' \neq \emptyset$.
\end{itemize}
\end{theorem}


\begin{proof}
(Shortening insensitive $\lang$)$\lang \cap Red_I(\lang') = \emptyset$ so there does not exists $r' \in \lang \cap Red_I(\lang')$. Because $\lang$ is shortening insensitive, it is impossible that any run $r$ with $r' \prec r$ belongs to $\lang \cap \lang'$.
(Lengthening insensitive $\lang$) At least one word $r'$ is in $\lang$ and $Red_I(\lang)$. Therefore the longer word $r$ of $\lang'$ that $r'$ represents is also in $\lang$
since the language is lengthening insensitive.

\qed
\end{proof}

With this theorem original to this paper we now can build a semi-decision procedure that is able to prove \emph{some} lengthening or disprove \emph{some} shortening insensitive properties using a reduction of a system. 

\section{Application to Verification}
\label{sec:red}
We now introduce the more concrete setting of LTL verification to exploit the theoretical results on languages and their shortening/lengthening sensitivity developed in Section~\ref{sec:defs}.

\subsection{Kripke Structure}
\label{ss:kripke}

From the point of view of LTL verification with a state-based logic, executions of a system (also called \emph{runs}) are seen as infinite words over the alphabet $\Sigma=2^{\AP}$, where $\AP$ is a set of atomic propositions that may be true or false in each state. So each symbol in a run gives the truth value of all of the atomic propositions in that state of the execution, and each time an action happens we progress in the run to the next symbol. Some actions of the system update the truth value of atomic propositions, but some actions can leave them unchanged which corresponds to stuttering. 

\begin{definition}
[Kripke Structure Syntax] Let $\AP$ designate a set of atomic propositions. 
A Kripke structure $\KS_\AP=\tuple{S,R,\lambda,s_0}$ over $\AP$ is a tuple where $S$ is the finite set of states, $R \subseteq S \times S$ is the transition relation, $\lambda : S \mapsto 2^\AP$ is the state labeling function, and $s_0 \in S$ is the initial state.
\end{definition}
\begin{definition}[Kripke Structure Semantics]
\label{def:kripke}
The language $\lang(\KS_\AP)$ of a Kripke structure $\KS_\AP$ is defined over the alphabet $2^\AP$.
It contains all runs of the form $r=\lambda(s_0) \lambda(s_1) \lambda(s_2)\ldots$ where $s_0$ is the initial state of $\KS_\AP$ and $\forall i \in \nat$, either $(s_i,s_{i+1}) \in R$, or if $s_i$ is a deadlock state such that $\forall s' \in S, (s_i,s') \not\in R$ then $s_{i+1}=s_i$.
\end{definition}

All system executions are considered maximal, so that they are represented by infinite runs. If the system can deadlock or terminate in some way, we can extend these finite words by an $\omega$ stutter on the last symbol of the run to obtain a run.

Subfigure (1) of Figure~\ref{fig:example} depicts a program where each thread ($\alpha$ and $\beta$) has  three reachable positions (we consider that each instruction is atomic).
In this example we consider that the logic only observes two atomic propositions $p$ (true when $x=0$) and $q$ (true when $y=0$). The variable $z$ is not observed.

Subfigure (2) of Figure~\ref{fig:example} depicts the reachable states of this system as a Kripke structure.
Actions of thread $\beta$ (which do not modify the value of $p$ or $q$) are horizontal while actions of thread $\alpha$ are vertical. While each thread has 3 reachable positions, the emission of the message by $\beta$ must precede the reception by $\alpha$ so that some situations are unreachable.
Based on Definition~\ref{def:kripke} that extends by an infinite stutter runs that end in a deadlock, we can compute the language $\lang_A$ of this system. It consists in three parts: when thread $\beta$ goes first $pq^3$ $\bar p q$ $\bar p \bar q ^\omega$, with an interleaving $pq^2$ $\bar p q^2$ $\bar p \bar q ^\omega$, and when thread $\alpha$ goes first $pq$ $\bar p q^3$ $\bar p \bar q ^\omega$.

In subfigure (3) of Figure~\ref{fig:example}, 
the actions "$z=40;chan.send(z);$"of thread $\beta$ are fused into a single atomic operation.
This is possible because action $z=40$ of thread $\beta$ is stuttering (it cannot affect either $p$ or $q$) and is non-interfering with other events (it neither enables nor disables any event other than subsequent instruction "chan.send(z)"). 
 The language of this smaller KS is a reduction of the language of the original system. It contains two runs: thread $\alpha$ goes first  $pq$ $\bar p q^2$ $\bar p \bar q ^\omega$ and thread $\beta$ goes first $pq^2$ $\bar p q$ $\bar p \bar q ^\omega$.

In subfigure (4) of Figure~\ref{fig:example}, 
the already fused action "$z=40;chan.send(z);$"of thread $\beta$ is further fused with the \textit{chan.recv();} action of thread $\alpha$.
This leads to a smaller KS whose language is still a reduction of the original system now containing a single run: $pq$ $\bar p q$ $\bar p \bar q ^\omega$.
This simple example shows the power of structural reductions when they are applicable, with a drastic reduction of the initial language.

\subsection{Automata theoretic verification}

Let us consider the problem of model-checking of an $\omega$-regular  property $\varphi$ (such as LTL) on a system using the automata-theoretic approach~\cite{Vardi07}. In this approach, we wish to answer the problem of language inclusion: do all runs of the system $\lang(\KS)$ belong to the language of the property $\lang(\varphi)$ ? To do this, when the property $\varphi$ is an omega-regular language (e.g. an LTL or PSL formula), we first negate the property $\lnot\varphi$, then build a (variant of) a Büchi automaton $A_{\lnot\varphi}$ whose language\footnote{Because computing the complement $\bar{A}$ of an automaton $A$ is exponential in the worst case, syntactically negating $\varphi$ and producing an automaton $A_{\lnot \varphi}$ is preferable when $A$ is derived from e.g. an LTL formula.} 
 consists of runs that are not in the property language $\lang(A_{\lnot\varphi}) = \Sigma^\omega \setminus \lang(\varphi)$.
We then perform a synchronized product between this Büchi automaton and the Kripke structure $\KS$ corresponding to the system's state space $A_{\lnot\varphi} \otimes \KS$ (where $\otimes$ is defined to satisfy $\lang(A \otimes B)=\lang(A) \cap \lang(B)$). Either the language of the product is empty $\lang(A_{\lnot\varphi} \otimes \KS) = \emptyset$, and the property $\varphi$ is thus true of this system, or the product is non empty, and from any run in the language of the product we can build a counter-example to the property.

We will consider in the rest of the paper that the shortening or lengthening  insensitive language of definitions~\ref{def:shortins} and~\ref{def:longins} is given as an omega-regular language or Büchi automaton typically obtained from the negation of an LTL property, and that the reduction of definition~\ref{def:reduction} is applied to a language that corresponds to all runs in a Kripke structure typically capturing the state space of a system.

\medskip
\textbf{LTL verification with reductions.} With theorem~\ref{th:short}, a shortening insensitive property shown to be true on the reduction (empty intersection with the language of the negation of the property) is also true of the original system. A lengthening insensitive property shown to be false on the reduction (non-empty intersection with the language of the negation of the property, hence counter-examples exist) is also false in the original system.
Unfortunately, our procedure cannot prove using a reduction that a shortening insensitive property is false, or that a lengthening insensitive property is true. We offer a semi-decision procedure.

\subsection{Detection of language sensitivity}

We now present a strategy to decide if a given property expressed as a Büchi automaton is shortening insensitive, lengthening insensitive, or both.

This section relies heavily on the operations introduced and discussed at length in~\cite{ADL15}. 
The authors define two \emph{syntactic} transformations $sl$ and $cl$ of a transition-based generalized Büchi automaton (TGBA) $A_\varphi$ that can be built from any LTL formula $\varphi$ to represent its language $\lang(\varphi)=\lang(A_\varphi)$~\cite{CouvreurFM99}. TGBA are a variant of Büchi automata where the acceptance conditions are placed on edges rather than states of the automaton. 

The $cl$ closure operation \emph{decreases stutter}, it adds to the language any word $r' \in \Sigma^\omega$ that is shorter than a word $r$ in the language. Informally, the strategy consists in detecting when a sequence $q_1 \tr{a} q_2 \tr{a} q_3$ is possible and adding an edge $q_1 \tr{a} q_3$, hence its name $cl$ for "closure". 
The $sl$ self-loopization operation \emph{increases stutter}, it adds to the language any run $r' \in \Sigma^\omega$ that is longer than a run $r$ in the language.
Informally, the strategy consists in adding a self-loop to any state labeled with all outgoing expressions so that we can always decide to repeat a letter rather than progress in the automaton, hence its name $sl$ for "self-loop". 
More formally $\lang(cl(A_\varphi))=\{r' \mid \exists r \in \lang(A_\varphi), r' \pre r \}$ and $\lang(sl(A_\varphi))=\{r' \mid \exists r \in \lang(A_\varphi), r \pre r' \}$.

Using these operations \cite{ADL15} shows that there are several possible ways to test if an omega-regular language (encoded as a Büchi automaton) is stutter insensitive: 
essentially applying either of the operations $cl$ or $sl$ should leave the language unchanged. This allows to recognize that a property is stutter insensitive even
though it syntactically contains e.g. the neXt operator of LTL.

For instance $A_\varphi$ is stutter insensitive if and only if $\lang(sl(cl(A_\varphi)) \otimes A_{\lnot \varphi}) =\emptyset$.
The full test is thus simply reduced to a language emptiness check testing that both $sl$ and $cl$ operations
leave the language of the automaton unchanged.

Indeed for stutter insensitive languages, all or none of the runs belonging to a given stutter equivalence class of runs $\hat{r}$ must belong to the language $\lang(A_\varphi)$. 
In other words, if shortening or lengthening a run can make it switch from belonging to $A_\varphi$ to belonging to $A_{\lnot \varphi}$, the language is stutter sensitive. 
This is apparent on Figure~\ref{fig:lang}

We want weaker conditions here, but we can reuse the $sl$ and $cl$ operations developed for testing stutter insensitivity.
Indeed for an automaton $A$ encoding a shortening insensitive language, $\lang(cl(A))=\lang(A)$ should hold.
Conversely if $A$ encodes a lengthening insensitive language, $\lang(sl(A))=\lang(A)$ should hold.
We express these tests as emptiness checks on a product in the following way.

\begin{theorem}
[Testing sensitivity]
Let $A$ designate a Büchi automaton, and $\bar{A}$ designate its complement.

$\lang(cl(A) \otimes \bar{A}) = \emptyset$, if and only if $A$ defines a shortening insensitive language.

$\lang(sl(A) \otimes \bar{A}) = \emptyset$ if and only if $A$ defines a lengthening insensitive language.
\end{theorem}

\begin{proof}
The expression $\lang(cl(A) \otimes \bar{A}) = \emptyset$ is equivalent to $\lang(cl(A))=\lang(A)$. The lengthening insensitive case is similar.
\qed
\end{proof}

Thanks to property~\ref{prop:shortduallong}, and in the spirit of~\cite{ADL15} we could also test the complement of a language for the dual property if that is more efficient, i.e. $\lang(sl(\bar{A}) \otimes A) = \emptyset$ if and only if $A$ defines a shortening insensitive language and similarly $\lang(cl(\bar{A}) \otimes A) = \emptyset$ iff $A$ is lengthening insensitive. We did not really investigate these alternatives as the complexity of the test was already negligible in all of our experiments.


\subsection{Agglomeration of events produces shorter runs}

Among the possible strategies to reduce the complexity of analyzing a system are structural reductions.
Depending on the input formalism the terminology used is different, but the main results remain stable.

In~\cite{Lipton75} transaction reduction consists in fusing two adjacent actions of a thread (or even across threads in recent versions such as~\cite{Laarman18} ). 
The first action must not modify atomic properties and must be commutative with any action of other threads. Fusing these actions leads to shorter runs, where a stutter is lost. 
In the program of Fig.~\ref{fig:example}, "z=40" is enabled from the initial state and must happen before "chan.send(z)", but it commutes with instructions of thread $\alpha$ and is not observable. 
Hence the language $\lang_B$ built with an atomic assumption on "z=40;chan.send(z)" is indeed a reduction of $\lang_A$.

Let us reason at the level of a Kripke structure.
The goal of such reductions is to structurally detect the following situation in language $\lang$: let $r=a_0^{n_0}a_1^{n_1}a_2^{n_2}\ldots$ designate a run (not necessarily in the language), there must exist two indexes $i$ and $j$ such that for any natural number $k$, $i \leq k \leq j$, $r_k=a_0^{n_0}\ldots a_i^{n_i}\ldots a_k^{n_k + 1} \ldots a_j^{n_j} \ldots$ is in the language. In other words, the set of runs described as : $\{ r_k=a_0^{n_0}\ldots a_i^{n_i}\ldots a_k^{n_k + 1} \ldots a_j^{n_j} \ldots \mid i \leq k \leq j\}$ must belong to the language. 
This corresponds to an event that does not impact the truth value of atomic propositions (it stutters) and can be freely commuted with any event that occurs between indexes $i$ and $j$ in the run. This event is simply constrained to occur at the earliest at index $i$ in the run and at the latest at index $j$. In Fig~\ref{fig:example} the event "z=40" can happen as early as in the initial state, and must occur before "chan.send(z)" and thus matches this definition.

Note that these runs are all stutter equivalent, but are incomparable by the shorter than relation (e.g. $aabc, abbc, abcc$ are incomparable). In this situation, a \emph{reduction} can choose to only represent the run $r$ instead of any of these runs. This run was not originally in the language in general, but it is indeed shorter than any of the $r_k$ runs so it matches definition~\ref{def:reduction} for a reduction. Note that $\hat{r}$ does contain all these longer runs so that in a stutter insensitive context, examining $r$ is enough to conclude for any of these runs. This is why usage of structural reductions is compatible with verification of a logic such as $LTL_{\setminus X}$ and has been proposed for that express purpose in the literature~\cite{HPP06,Laarman18}.

Thus transaction reductions~\cite{Lipton75,Laarman18} as well as  
both pre-agglomeration and post-agglomeration of Petri nets~\cite{PPP00,EHPP05,HPP06,YTM20} produce a system whose language is a reduction of the language of the original system. 

For lack of space, in this paper we decided not to provide proofs that these structural transformation rules induce reductions at the language level.
A formal definition involves a) introducing the syntax of a formalism and b) its semantics in terms of language, then c) defining the reduction rule, and d) proving its effect  is a reduction at the language level. The exercise is not particularly difficult, and the definition of reduction rules mostly fall into the category above, where a non observable event that happens at the earliest at point $a$ and at the latest at point $b$ is abstracted from the trace.

Our experimental Section~\ref{ss:modelchecking} uses the rules of~\cite{YTM20} for (potentially partial) pre and post-agglomeration.
That paper presents $22$ structural reductions rules from which we selected the rules valid in the context of LTL verification. Only one rule preserving stutter insensitive LTL was not compatible with our approach since it does not produce a reduction at the language level: rule $3$ "Redundant transitions" proposes that if two transitions $t_1$ and $t_2$ have the same combined effect as a transition $t$, and firing $t_1$ enables $t_2$, $t$ can be discarded from the net. This reduces the number of edges in the underlying \KS{} representing the state space, but does not affect reachability of states. However, it selects as representative a run involving both $t_1$ and $t_2$ that is longer than the one using $t$ in the original net, it is thus not legitimate to use in our strategy (although it remains valid for $LTL_{\setminus X}$). 
Rules $14$ "Pre agglomeration" and $15$ "Post agglomeration" are the most powerful rules of~\cite{YTM20} that we are able to apply in our context. They are known to preserve $LTL_{\setminus X}$ (but not full LTL) and their effect is a reduction at the language level, hence we \emph{can} use them when dealing with shortening/lengthening insensitive formulae. 

\section{Experimentation}
\label{sec:perfs}

\subsection{A Study of Properties}

This section provides an empirical study of the applicability of the techniques presented in this paper to LTL properties found in the literature. To achieve this we explored several LTL benchmarks~\cite{etessami.00.concur,somenzi.00.cav,dwyer.98.fmsp,rers21,mcc:2021}.
Some work~\cite{etessami.00.concur,somenzi.00.cav} summarises  the typical properties that users express in LTL. The formulae of this benchmark have been extracted  directly from the literature. 
Dwyer et al.~\cite{dwyer.98.fmsp}  proposes property specification patterns, expressed in several logics including LTL. These patterns have been extracted by analysing 447 formulae coming from real world projects.
The RERS challenge~\cite{rers21} presents generated formulae inspired from real world reactive systems. 
    The MCC~\cite{mcc:2021} benchmark establishes a huge database of $45152$ LTL formulae in the form of $1411$ 
    Petri net models coming from $114$ origins with for each one $32$ random LTL formulae. These formulae use up to $5$ state-based atomic propositions, limit the nesting depth of temporal operators to $5$ and are filtered in order to be non trivial. Since these formulae come with a concrete system we were able to use this benchmark to also provide performance results for our approach in Section~\ref{ss:modelchecking}. We retained $43989$ model/formula pairs from this benchmark, the missing $1163$ were rejected due to parse limitations of our tool when the model size is excessive ($>10^7$ transitions).
This set of roughly 2200 human-like formulae and 44k random ones lets us evaluate if the fragment of LTL that we consider is common in practice. 
Table~\ref{tab:formulae} summarizes, for each benchmark, the number and percentage of formulae that are either stuttering insensitive, lengthening insensitive, or shortening insensitive. 
The sum of both shortening and lengthening  formulae represents more than one third (and up to 60 percent) of the formulae of these benchmarks. 

Concerning the polarity, although lengthening insensitive formulae seem to appear more frequently, 
most of these benchmarks actually contain each formula in both positive and negative forms (we retained only one) so that the summed percentage might be more relevant as a metric since lengthening insensitivity of $\varphi$ is equivalent to shortening insensitivity of $\lnot\varphi$.
Analysis of the human-like Dwyer patterns~\cite{dwyer.98.fmsp} reveals that shortening/lengthening insensitive formulae mostly come from the patterns \textit{precedence chain}, \textit{response chain} and \textit{constrained chain}.
These properties specify causal relation between events, that are observable as causal relations 
between \emph{observably different} states (that might be required to strictly follow each other), but this causality chain is not impacted by non observable events.


\begin{figure}[p]

    \begin{tabular}{p{2cm}p{2cm}p{2cm}p{2cm}p{2cm}p{2cm}}
        Benchmark       & Stutter Insens. & Length. Insens. & Short. Insens. & Others & Total\\
                        \cmidrule(r){1-1}\cmidrule(r){2-2}
                        \cmidrule(r){3-3}\cmidrule(r){4-4}\cmidrule(r){5-5}
                        \cmidrule(r){6-6}
         Spot~\cite{etessami.00.concur,dwyer.98.fmsp,somenzi.00.cav}   & 63 (67\%)    & 17 (18\%)   & 11 (1\%)   & 3 (3\%)   & 94 \\
         Dwyer et al.~\cite{dwyer.98.fmsp}  & 32 (58\%)     & 13  (23\%) & 9 (16\%)   & 1 (1.81\%)     & 55\\
         RERS~\cite{rers21}.  & 714 (35\%)  & 777  (38\%) & 559 (28\%) & 0 & 2050\\
         MCC~\cite{mcc:2021}    & 24462 (56\%) & 6837 (14\%) &5390 (12\%) & 7300 (16\%) & 43989 \\~\\
    \end{tabular}
    \captionof{table}{Sensitivity to length of properties measured using several LTL benchmarks.}
    \label{tab:formulae}

\vspace{2ex}

  \begin{subfigure}[t]{0.45\textwidth}
    \includegraphics[width=\textwidth]{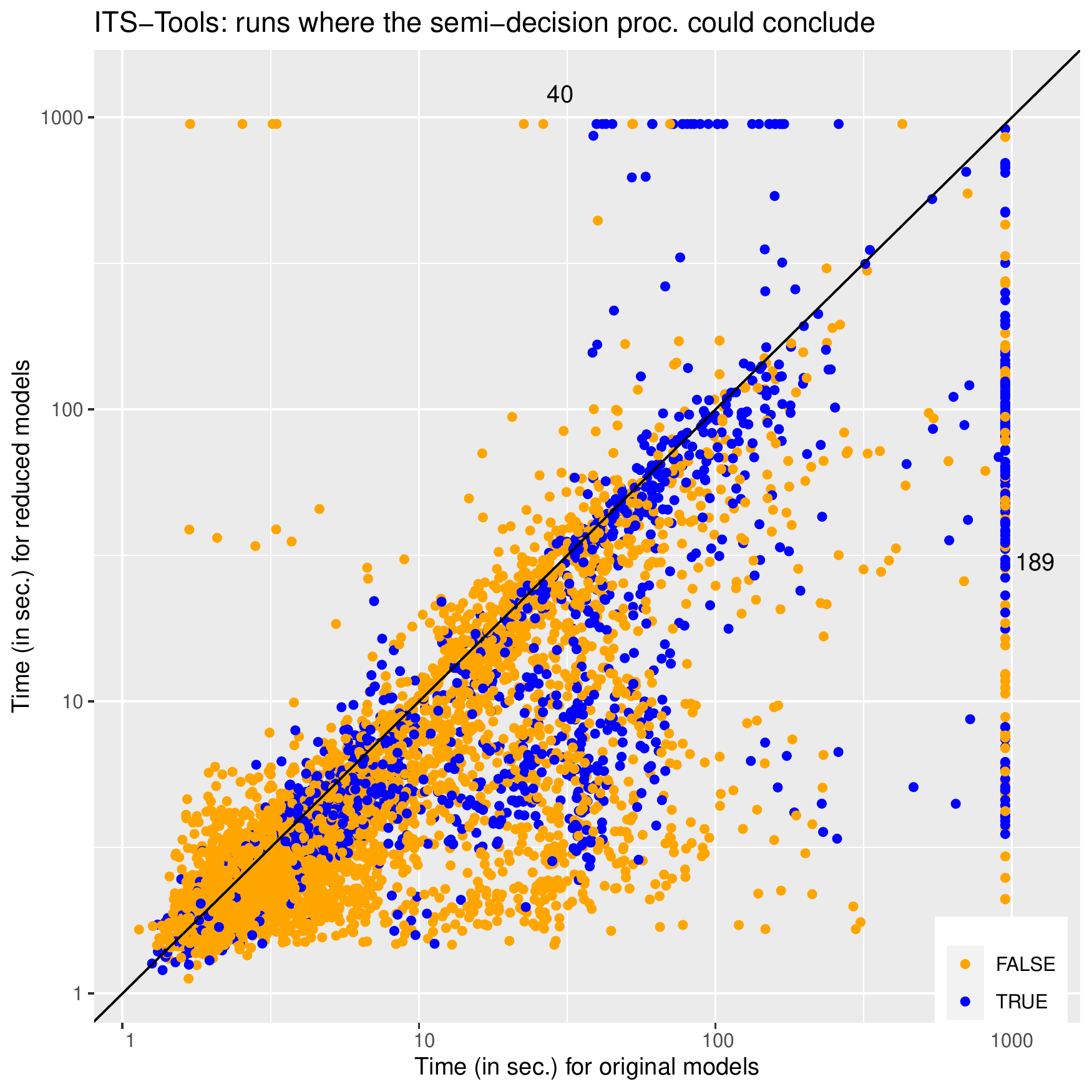}
    \caption{ITS-tools on decidable instances}
    \label{fig:1}
  \end{subfigure}
 \hfill
  \begin{subfigure}[t]{0.45\textwidth}
    \includegraphics[width=\textwidth]{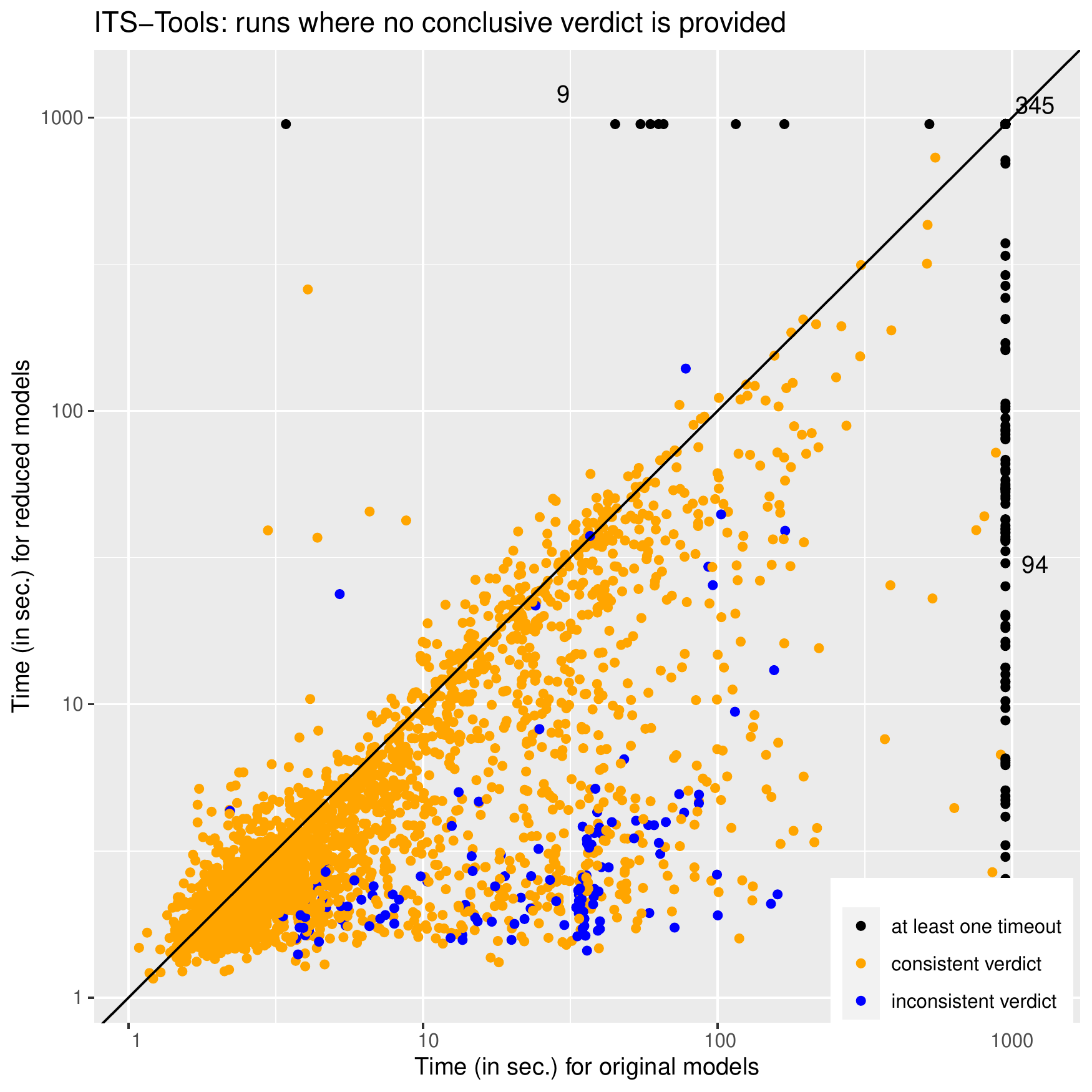}
    \caption{ITS-tools on non decidable instances}
    \label{fig:2}
  \end{subfigure}
  \begin{subfigure}[t]{0.45\textwidth}
    \includegraphics[width=\textwidth]{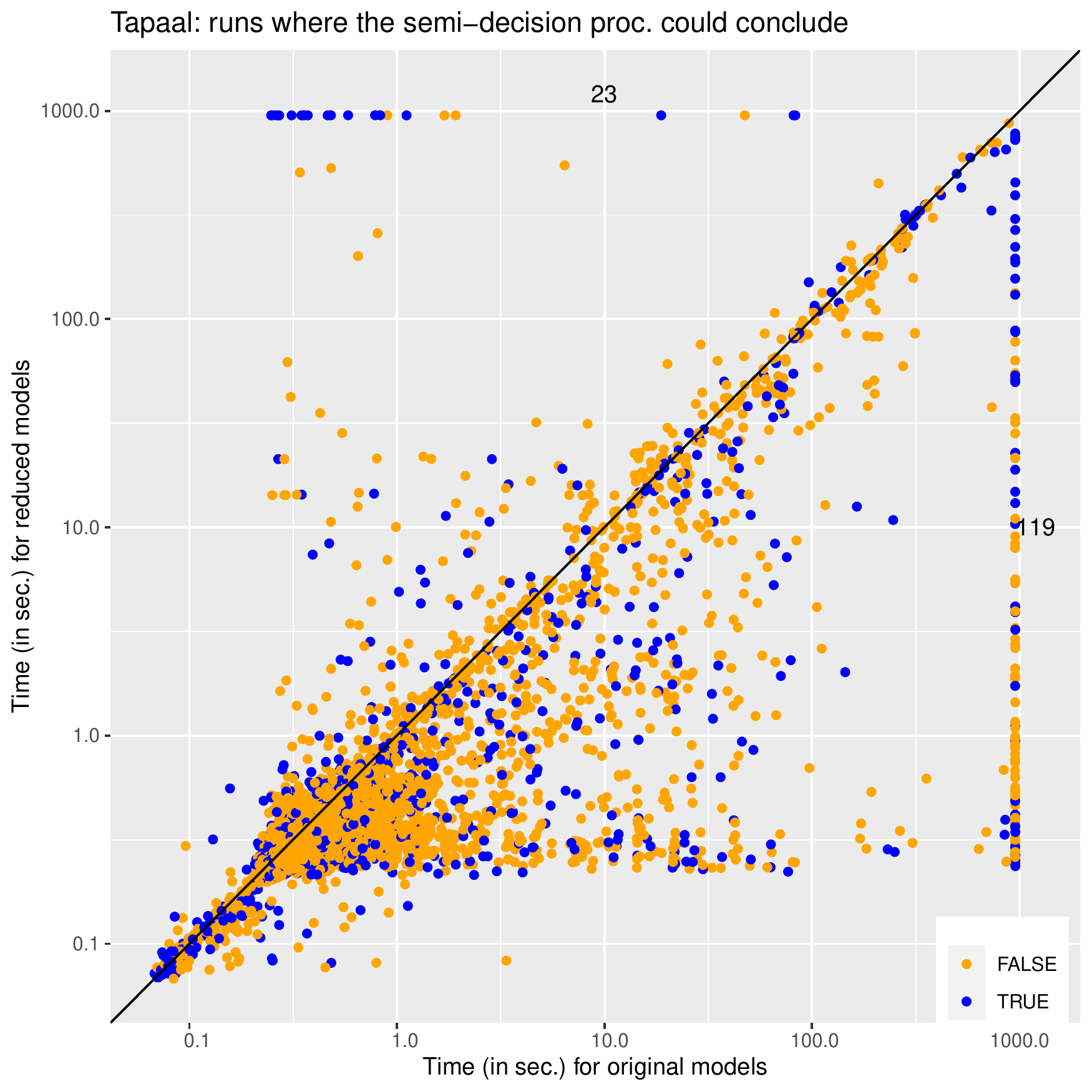}
    \caption{Tapaal on decidable instances}
    \label{fig:3}
  \end{subfigure}
 \hfill
  \begin{subfigure}[t]{0.45\textwidth}
    \includegraphics[width=\textwidth]{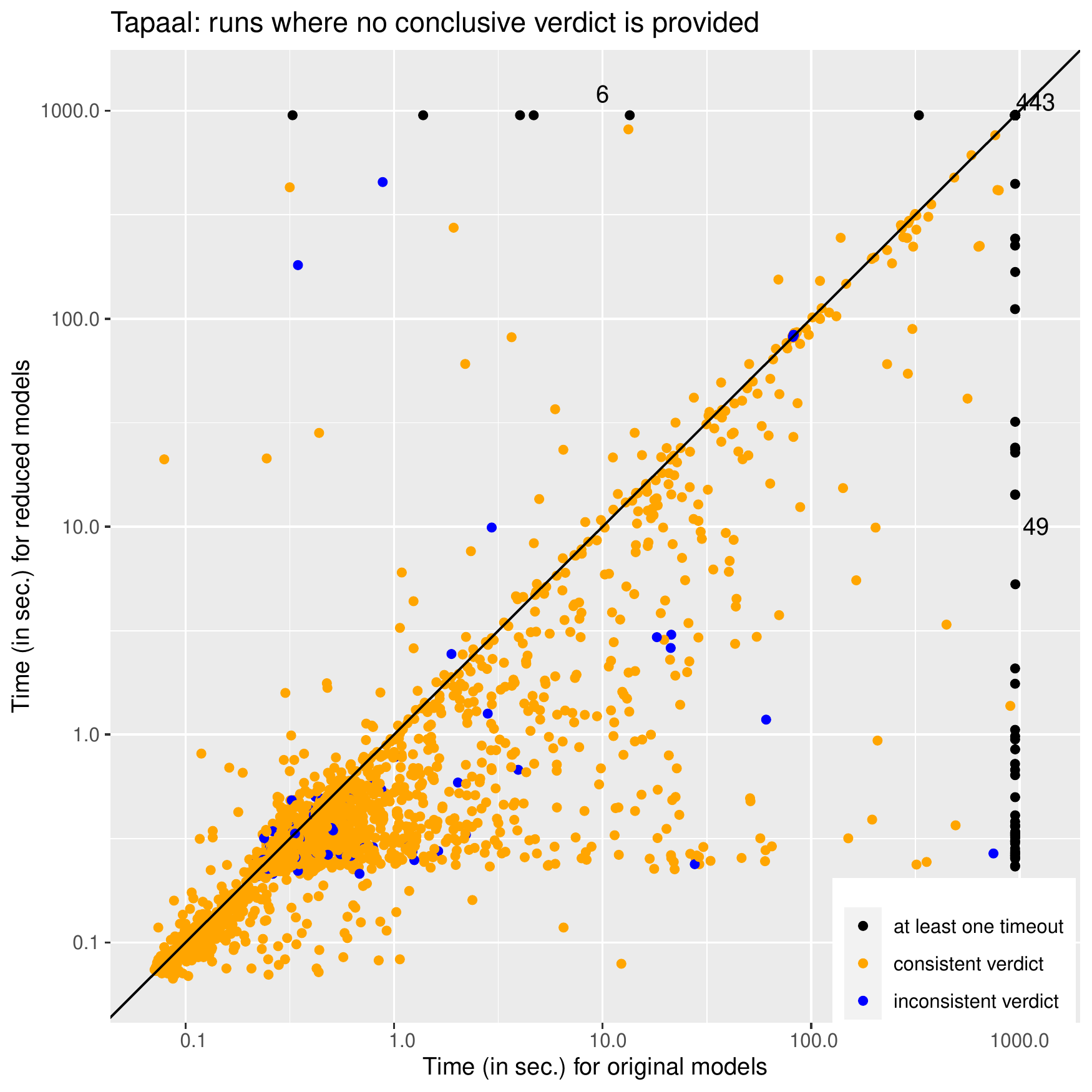}
    \caption{Tapaal on non decidable instances}
    \label{fig:4}
  \end{subfigure}
  \caption{Experiments on the MCC'2021 LTL benchmark using the two best tool of the MCC contest: Tapaal and ITS-tools. 
Figures (a) and (c) contain the cases where the verdict of the semi-decisions procedures is reliable, 
and distinguish cases where the output is True (empty product) and False (non empty product). 
(b) and (d) display the cases where the verdict is not reliable and distinguish cases where the output is
inconsistent with the ground truth from cases where they agree. 
}
\label{fig:benchmark}

\end{figure}

\subsection{A Study of Performances}
\label{ss:modelchecking}

\textbf{Benchmark Setup.}
Among the LTL benchmarks presented in Table~\ref{tab:formulae}, we opted for the MCC benchmark 
to evaluate the techniques presented in this paper. 
This benchmark seems relevant since 
(1) it contains both academic and industrial models, 
(2) it has a huge set of (random) formulae and
(3)  includes models so that we could measure the effect of the approach in a model-checking setting. 
The model-checking competition (MCC) is an annual event in its $10^{th}$ edition in 2021 where competing tools 
are evaluated on a large benchmark. 
We use the formulae and models from the latest $2021$ edition of the contest, where Tapaal~\cite{tapaal12} was awarded the gold medal
and ITS-Tools~ \cite{ITStools} was silver in the LTL category of the contest.
We evaluate both of these tools in the following performance measures, showing that our strategy is agnostic
to the back-end analysis engine. Our experimental setup consists in two steps.
\begin{enumerate}
    \item Parse the model and formula pair, and analyze the sensitivity of the formula. When the formula is 
    shortening or lengthening insensitive (but not both) output two model/formula pairs : reduced and original.
    The "original" version does also benefit from reduction rules, but we apply only rules that are compatible with full LTL.
    The "reduced" version additionally benefits from rules that are reductions at the language level, i.e. mainly pre and post agglomeration (but enacting these rules can cause further simplifications). The original and reduced model/formula pairs
    that result from this procedure are then exported in the same format the contest uses. This step was implemented within ITS-tools.
    \item Run an MCC compatible tool on both the reduced and original versions of each model/formula pair and record the time performance and the verdict. 
\end{enumerate}

For the first step, using Spot~\cite{ADL15}, we detect that a formula is either shortening or lengthening insensitive for 99.81\% of formulae in less than 1 second. After this analysis, we obtain 12\,227 model/formula pairs where the formula is either shortening insensitive or lengthening insensitive (but not both). Among these pairs, in $3005$ cases (24.6\%) the model was resistant to the structural reduction rules we use. 
Since our strategy does not improve such cases, we retain the remaining 9\,222 (75.4\%) 
model/formula pairs in the performance plots of Figure~\ref{fig:benchmark}. 
We measured that on average 34.19\% of places and 32.69\% of the transitions of the models were discarded by reduction rules with respect to the "original" model, though
the spread is high as there are models that are almost fully reducible and some that are barely so.
Application of reduction rules is in complexity related to the size of the structure of the net and 
takes less than 20 seconds to compute in 95.5\% of the models.
We are able to treat $9222$ examples (21\% of the original $43989$ model/formula pairs of the MCC) using reductions. All these formulae until now could not be handled using reduction techniques.

For the second step, we measured the solution time for both reduced and original model/formula pairs using the two best tools of the MCC'2021 contest. 
A full tool using our strategy might optimistically first run on the reduced model/formula pair hoping for a definitive answer, but we recommend the use of a portfolio approach where the first reliable answer is kept.
In these experiments we neutrally measured the time for taking a semi-decision on the reduced model vs. the time for taking a (complete) decision on the original model. We then classify the results into two sets, decidable instances are shown on the left of Fig.~\ref{fig:benchmark} and instances that are not decidable (by our procedure) are on the right. 
On "decidable instances" our semi-decision procedure could have concluded reliably because the formula is true and the property shortening insensitive or the formula is false and the property lengthening insensitive. 
Non decidable instances shown on the right are those where the verdict on the reduced model is not to be trusted
(or both the original and reduced procedures timed out).

With this workflow we show that our approach is generic and can be easily implemented on top of any MCC compatible model-checking tool. All experiments were run with a 950 seconds timeout (close to 15 minutes, which is generous when the contest offers 1 hour for 16 properties).
We used a heterogeneous cluster of machines with four cores allocated to each experiment, and ensured that experiments concerning reduced and original versions of a given model/formula 
are comparable. 

Figure~\ref{fig:benchmark} presents the results of these  experiments.
The results are all presented as log-log scatter plots opposing a run on the original to a run on the reduced model/formula pair.
Each dot represents an experiment on a model/formula pair; a dot below the diagonal indicates that the reduced version was faster to solve, while a point above it indicates a case where the reduced model actually took longer to solve than the original (fortunately there are relatively few). 
Points that timeout for one (or both) of the approaches are plotted on the line at 950 seconds, we also indicate the number of points that are in this line (or corner) next to it.

The plots on the left  (a) and (c) correspond to "decidable instances" while those on the right are not decidable by our procedure.
The two plots on the top correspond to the performance of ITS-tools, while those on the bottom give the results with Tapaal.  
The general form of the results with both tools is quite similar confirming that our strategy is indeed responsible for the measured gains in performance and that they are reproducible. Reduced problems \emph{are}  generally easier to solve than the original.
This gain is in the best case exponential as is visible through the existence of spread of points reaching out horizontally in this log-log space (particularly on the Tapaal plots).

The colors on the decidable instances reflect whether the verdict was true or false.
For false properties a counter-example was found by both procedures interrupting the search, and while the search space of
a reduced model is a priori smaller, heuristics and even luck can play a role in finding a counter-example early.
True answers on the other hand generally require a full exploration of the state space so that the reductions should play a major role in reducing the complexity of model-checking. The existence of True answers where the reduction fails is surprising at first, but a smaller Kripke structure does not necessarily induce a smaller product as happens sometimes in this large benchmark (and in other reduction techniques such as stubborn sets~\cite{Valmari90}). 
On the other hand the points aligned to the right of the plots a) and c) (189 for ITS-tools and 119 for Tapaal) 
correspond to cases where our procedure improved these state of the art tools, allowing to reach a conclusion when the original method fails.

The plots on the right use orange to denote cases where the verdict on the reduced and original models were the same; on these points
the procedures had comparable behaviors (either exploring a whole state space or exhibiting a counter-example). The blue color denotes points where the two procedures disagree, with several blue points above the diagonal reflecting cases where the reduced procedure explored the whole state space and thought the property was true while the original procedure found a counter-example (this is the worst case).
Surprisingly, even though on these non decidable plots b) and d) our procedure should not be trusted, it mostly agrees (in 95\% of the cases) with the decision reached on the original.

Out of the 9222 experiments in total, for ITS-tools 5901 runs reached a trusted decision (64 \%), 2927 instances reached an untrusted verdict (32 \%), and the reduced procedure timed out in 394 instances (4 \%). Tapaal reached a trusted decision in 5866 instances (64 \%), 2884 instances reached an untrusted verdict (31 \%), and the reduced procedure timed out in 472 instances (5 \%). On this benchmark of formulae we thus reached a trusted decision in almost two thirds of the cases using the reduced procedure.

\section{Related Work}
\label{sec:related}

\noindent\textbf{Partial order vs structural reductions.}
Partial order reduction (POR)~\cite{porBook,Valmari90,Peled94,GodefroidW94} is a very popular approach to combat state explosion for 
\emph{stutter insensitive} formulae.
These approaches use diverse strategies (stubborn sets, ample sets, sleep sets\ldots) to consider only a subset of events at each step of the model-checking while still ensuring that at least one representative of each stutter equivalent class of runs is explored.
Because the preservation criterion is based on equivalence classes of runs, this family of approaches is limited only to the stutter insensitive fragment of LTL (see Fig.\ref{fig:lang}). 
However the structural reduction rules used in this paper are compatible and can be stacked with POR when the formula is stutter insensitive; this is the setting in which most structural reduction rules were originally defined.

\noindent\textbf{Structural reductions in the literature.}
The structural reductions rules we used in the performance evaluation are defined on Petri nets where 
the literature on the subject is rich~\cite{Berthelot85,PPP00,EHPP05,HPP06,berthomieu19,YTM20}.
However there are other formalism where similar reduction rules have been defined such as~\cite{PR08} using "atomic" blocks in Promela,
transaction reductions for the widely encompassing intermediate language PINS of LTSmin~\cite{Laarman18}, and even in the context of multi-threaded programs~\cite{FlanaganQ03}. All these approaches are structural or syntactic, they are run prior to model-checking per se.

\noindent\textbf{Non structural reductions in the literature.} Other strategies have been proposed that instead of structurally reducing the system, dynamically build an abstraction of the Kripke structure where less observable stuttering occurs. These strategies build a $\KS{}$ whose language \emph{is} a reduction of the language of the original $\KS{}$ (in the sense of Def.~\ref{def:reduction}), that can then be presented to the emptiness check algorithm with the negation of the formula. They are thus also compatible with the approach proposed in this paper.
Such strategies include the Covering Step Graph (CSG) construction of~\cite{Vernadat97} where a "step" is performed (instead of firing a single event) that includes several independent transitions. 
The Symbolic Observation Graph of~\cite{KP08} is another example where states of the original $\KS{}$ are computed (using BDDs) and aggregated as long as the atomic proposition values do not evolve; in practice it exhibits to the emptiness check only shortest runs in each equivalence class hence it is a reduction.

\section{Conclusion}

To combat the state space explosion problem that LTL model-checking meets, structural reductions have been proposed that 
syntactically compact the model so that it exhibits less interleaving of non observable actions.  Prior to this work, all of these approaches were limited to the stutter insensitive fragment of the logic. We bring a semi-decision procedure that widens the applicability of these strategies to  formulae which are \emph{shortening insensitive} or \emph{lengthening insensitive}. The experimental evidence presented shows that the fragment of the logic covered by these new categories is quite useful in practice. 
An extensive measure using the models, formulae and the two best tools of the model-checking competition 2021 shows that our strategy
can improve the decision power of state of the art tools, and confirm that in the best case an exponential speedup of the decision procedure can be attained. We also identified several other strategies that are compatible with our approach since they construct a reduced language. In further work we are investigating how non trusted counter-examples of the reduced model could be confirmed on the original model.

\bibliographystyle{splncs04}
\bibliography{biblio.bib}

\end{document}